\documentclass[10pt,letterpaper]{article}

\usepackage[preprint]{arxiv}

\usepackage[utf8]{inputenc} 
\usepackage[T1]{fontenc}    
\usepackage[hidelinks]{hyperref}       
\usepackage{url}            
\usepackage{booktabs}       
\usepackage{subcaption}
\usepackage{xspace}

\def\ie{\textit{i.e.}\xspace}

\def\eg{\textit{e.g.}\xspace}

\usepackage{multirow}
\usepackage{amsfonts}       
\usepackage{nicefrac}       
\usepackage{microtype}      
\usepackage{lipsum}         
\usepackage{graphicx}
\usepackage{natbib}
\usepackage{doi}

\usepackage{amsmath}
\usepackage{amssymb}
\usepackage{amsthm}

\usepackage{cleveref}       

\newtheorem{theorem}{Theorem}
\newtheorem{lemma}[theorem]{Lemma}

\theoremstyle{remark}
\newtheorem*{remark}{Remark}

\usepackage{algorithm} 
\usepackage{algpseudocode} 

\title{\Large \bf Top-$n\sigma$: Not All Logits Are You Need}

\newcommand{\AnD}{\qquad}

\begin{document}

\newcommand{\topp}{top-$p$ }
\newcommand{\minp}{min-$p$ }
\newcommand{\topk}{top-$k$ }

\date{}


\author{Chenxia Tang \AnD Jianchun Liu \AnD Hongli Xu \AnD Liusheng Huang\\
	School of Computer Science and Technology, University of Science and Technology of China\\
	Suzhou Institute for Advanced Research, University of Science and Technology of China
}

\maketitle

\thispagestyle{empty}

\begin{abstract}

Large language models (LLMs) typically employ greedy decoding or low-temperature sampling for reasoning tasks, reflecting a perceived trade-off between diversity and accuracy. We challenge this convention by introducing top-$n\sigma$, a novel sampling method that operates directly on pre-softmax logits by leveraging a statistical threshold. Our key insight is that logits naturally separate into a Gaussian-distributed noisy region and a distinct informative region, enabling efficient token filtering without complex probability manipulations. Unlike existing methods (\eg, top-$p$, min-$p$) that inadvertently include more noise tokens at higher temperatures, top-$n\sigma$ maintains a stable sampling space regardless of temperature scaling. We also provide a theoretical analysis of top-$n\sigma$ to better understand its behavior. The extensive experimental results across four reasoning-focused datasets demonstrate that our method not only outperforms existing sampling approaches but also surpasses greedy decoding, while maintaining consistent performance even at high temperatures. 
\end{abstract}

\section{Introduction}
Large Language Models (LLMs) have revolutionized natural language processing (NLP), demonstrating remarkable capabilities across various domains, from code generation~\citep{chen2021evaluating} to mathematical reasoning~\citep{lewkowycz2022solving}, and complex problem-solving~\citep{wei2022chain}. 
These advancements are largely driven by the models' sophisticated text generation mechanisms, which underpin their versatility in diverse applications.
Token sampling, as the first step in text generation, plays a crucial role in determining the quality of model outputs.
While sampling strategies offer the potential advantage of generating diverse outputs through multiple attempts, their effectiveness in reasoning-intensive tasks remains a significant challenge.

Conventional wisdom suggests that deterministic methods like greedy decoding often surpass stochastic sampling in tasks requiring precise reasoning. This perception arises from observations that existing sampling techniques prioritize diversity and reduced repetition over reasoning accuracy. For instance, probability-based methods such as temperature-scaling~\citep{ackley1985learning}, nucleus sampling~\citep{holtzman2019curious}, \topk~\citep{fan-etal-2018-hierarchical}, and \minp~\citep{nguyen2024min}, as well as entropy-based techniques like mirostat sampling~\citep{basu2020mirostat} and $\eta$-sampling~\citep{hewitt2022truncation}, are not specifically designed to maintain reasoning quality.
We argue that this limitation stems from these methods' inability to effectively filter out irrelevant tokens, leading to a trade-off between diversity and reasoning fidelity. 
Nonetheless, it is essential to recognize that LLMs are inherently probabilistic models; deterministic methods like greedy decoding introduce bias by failing to accurately reflect the model's learned distribution. 
This insight underscores the need for more sophisticated sampling techniques that can leverage LLMs' probabilistic nature while preserving robust reasoning capabilities.

While existing approaches primarily focus on manipulating probability distributions, we argue that examining the pre-softmax logits can reveal deeper insights into the model's generation process. 
As shown in Figure~\ref{fig:logits_distribution}, an intriguing observation is that logit distributions exhibit a highly regular pattern, typically comprising two distinct components: a Gaussian-like distribution of background tokens and a set of prominent outliers, respectively. 
Although the majority of tokens follow a Gaussian distribution, as shown in Figure~\ref{fig:probs_dis}, the outlier tail dominates the probability mass. 
We refer to these components as the \textbf{noisy region} and the \textbf{informative region}, respectively.
It is worth noting that the maximum logit deviates from the mean by more than 5$\sigma$ (the standard deviation), substantially exceeding typical statistical criteria for outlier identification.

\begin{figure}[t]
\begin{subfigure}[t]{0.48\linewidth}
    \centering
    \includegraphics[scale=0.4]{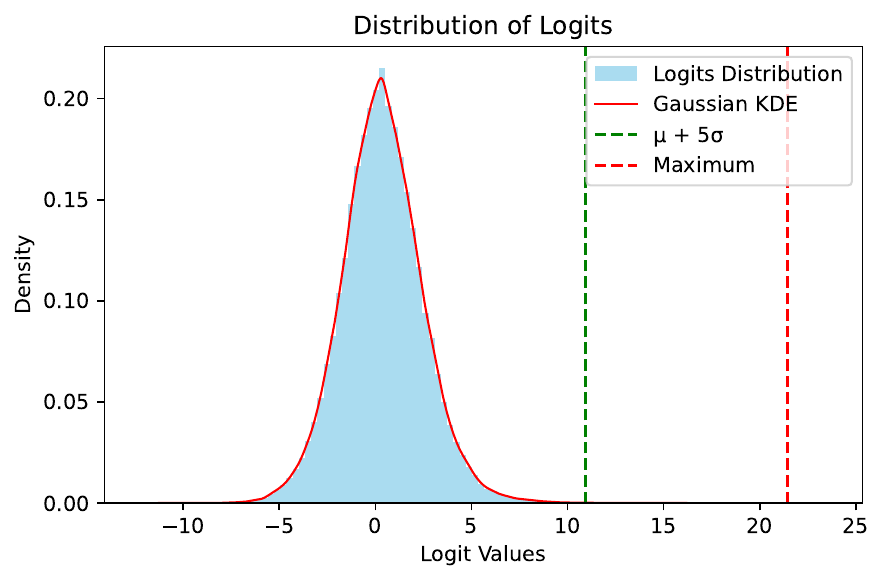}
    \caption{Distribution of logits}
    \label{fig:logits_dis}
\end{subfigure}
\hfill
\begin{subfigure}[t]{0.48\linewidth}
    \centering
    \includegraphics[scale=0.4]{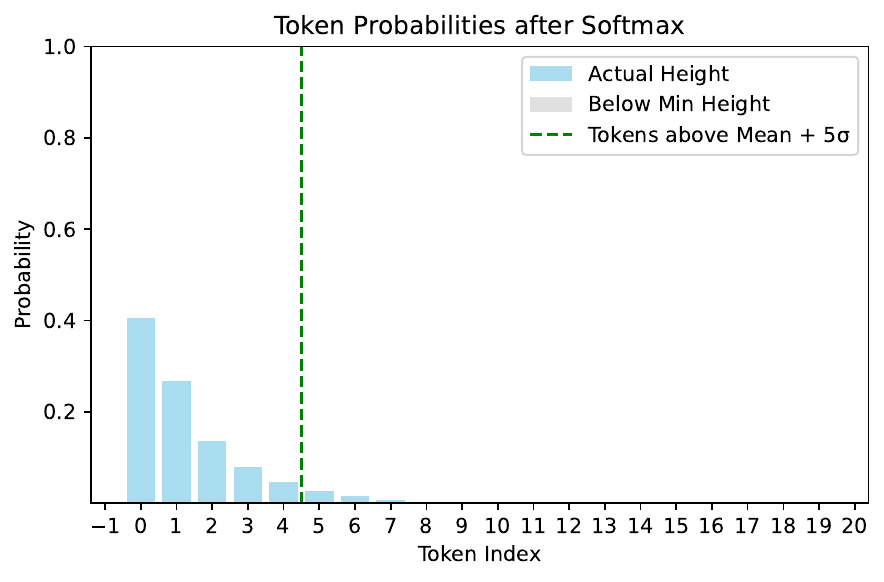}
    \caption{Descendingly sorted Probabilities. Only the top 20 tokens are shown. }
    \label{fig:probs_dis}
\end{subfigure}
    \caption{Distribution of logits and descendingly sorted probabilities of LLaMA3-8B-Instruct on an AQuA sample. Note that the leading tokens in the right plot (with higher probabilities) correspond to the right-side region of the logits distribution. The maximum logit is approximately $10\sigma$ above the mean of the distribution.}
    \label{fig:logits_distribution}
\end{figure}

A remarkable finding is that, despite not examining the logits distribution, accurately hypothesized that LLM output probabilities are a mixture of the true distribution and a smoothing component \citep{hewitt2022truncation}.
However, they mistakenly assume that this smooth component follows a uniform distribution. 
Our observation reveals that the mixture involves not a uniform distribution but a log-normal distribution, which naturally arises from applying softmax to the Gaussian-distributed logits in the noisy region.

Based on this observation, we propose top-$n\sigma$, a simple yet effective sampling method that naturally distinguishes between noisy and informative tokens. 
While this approach shares the spirit of nucleus sampling in identifying a core set of tokens (nucleus), top-$n\sigma$ offers several unique advantages. 
Unlike top-$k$ or greedy decoding, which imposes static constraints, or top-$p$ and min-$p$, which inadvertently include more irrelevant tokens at higher temperatures, top-$n\sigma$ maintains a stable and meaningful sampling space independent of temperature scaling.
This stability is particularly valuable for recent test-time scaling techniques that rely on extensive sampling to explore the solution space \citep{brown2024large, wang2024chain, qi2024mutual, openaio1}. 
Furthermore, top-$n\sigma$ operates directly on logits without requiring sorting operations or additional softmax transformations, making it computationally efficient. 
By identifying the informative tokens through statistical properties of the logit distribution, our method achieves both effectiveness and efficiency in guiding the sampling process.

Our main contributions include:

\begin{itemize}
    \item \textbf{Novel Logit-Based Perspective:} We introduce a novel analytical framework focusing on pre-softmax logit distributions, providing fundamental insights that benefit not only sampling strategy development but also potential improvements in model training approaches.
    
    \item \textbf{Efficient Top-$n\sigma$ Algorithm:} We propose a conceptually simple yet powerful sampling method that operates directly on logits, achieving superior generation quality while maintaining computational efficiency through its sorting-free and softmax-free implementation. It is very easy to integrate top-$n\sigma$ into existing LLM frameworks.

  \item \textbf{Ready for Test-Time Scaling Techniques:} Our top-$n\sigma$ algorithm enables a more granular exploration of the solution space, achieving a better balance between exploration and exploitation. This characteristic makes it particularly effective for test-time scaling techniques. 
    \item \textbf{Theoretical Analysis:} We provide comprehensive quantitative analyses of top-$n\sigma$, examining its cumulative probability mass characteristics and proving its temperature invariance property. These theoretical foundations establish a solid basis for both the implementation and understanding of the method.
    
    \item \textbf{Extensive Empirical Validation:} We demonstrate the effectiveness of our approach through rigorous experiments across four diverse datasets, showing significant improvements in generation quality compared to existing methods.
\end{itemize}

\section{Insights}

From Figure~\ref{fig:logits_distribution}, we observe that the model's logits appear to follow a well-defined statistical distribution. This empirical observation naturally raises several fundamental questions:

\begin{enumerate}
\item How can we interpret probability-based sampling methods (\eg, nucleus sampling) from the perspective of logit space?
\item What are the underlying characteristics of logit distributions in large language models?
\item How can we leverage these distributions to effectively distinguish between noisy and informative regions?
\end{enumerate}

In this section, we systematically address the aforementioned questions and present our findings.

\subsection{From Probabilities to Logits}

Modern large language models (LLMs) generate text through a two-step process: first producing logits, then converting them to probabilities via softmax transformation. To better understand sampling methods in the logit space, we start by examining how conventional probability-based methods, particularly nucleus sampling, can be reinterpreted from a logit perspective.

Given an input sequence, an LLM first generates a logit vector $l =(l_1,\cdots, l_V) \in \mathbb{R}^V$, where $V$ is the vocabulary size. These logits are firstly scaled by temperature: $l \leftarrow l/T$ and then transformed into probabilities $p=(p_1,\cdots, p_V) \in\mathbb{R}^V$ through the softmax function:

\begin{equation}
p_i = \frac{e^{l_i}}{s}, \quad \text{where } s = \sum_{j=1}^V e^{l_j}, 1\leq i \leq V
\end{equation}

Fundamentally, all sampling methods operate by determining a probability threshold $p^{(t)}\in [0, 1]$. Tokens with probabilities above this threshold form the sampling nucleus, and their cumulative probability defines the \textit{nucleus mass}. Formally, for a threshold $p^{(t)}$, the nucleus $\mathcal{N}$ is:

\begin{equation}
\mathcal{N} = \{i \mid p_i \geq p^{(t)}\} = \{i\mid l_i \geq t\}
\end{equation}

This probability-based selection can be equivalently performed in the logit space by determining a corresponding logit threshold $t = \ln(s\cdot p^{(t)})$. 

Based on our empirical observations in Figure~\ref{fig:logits_distribution}, we hypothesize that logits follow a certain distribution $f$. Under this assumption, we can derive the relationship between nucleus sampling parameters and logit threshold $t$. 

\begin{theorem}
\label{thrm:sum_to_int}
Consider $V$ logits $\{l_i\mid i=1, \cdots, V\}$ independently and identically distributed according to $f(x)$. For any threshold $t$, we have:

$$\sum_{l_i > t}e^{l_i} \overset{P}{\rightarrow} V\int_{t}^{+\infty} e^xf(x)\mathrm{\, d}x$$ 
\end{theorem}

The complete proof is provided in Appendix~\ref{app:proof_of_sum_to_int}. Although the theorem itself is conceptually simple, it gives rise to a series of powerful lemmas with practical implications: 

\begin{lemma}
    Denote $\mathcal{I}(t) = \int_{t}^{+\infty}e^xf(x)\mathrm{\, d}x$, and thus $s = V\cdot \mathcal{I}(-\infty)$.
    The nucleus mass of a given threshold $t$ is 

    \begin{equation}
    \label{eq:nucleus_mass}    
    p_{\mathcal{N}} = \sum_{i\in\mathcal{N}}p_i = \frac{\mathcal{I}(t)}{\mathcal{I}(-\infty)}
    \end{equation}

\end{lemma}

\begin{remark}
A key implication of this lemma is that the logit threshold for nucleus sampling can be analytically derived given only the probability distribution of the logits. 
We proceed to demonstrate this by deriving closed-form expressions for two fundamental cases. Given the cumulative probability threshold $p$:

\begin{enumerate}
\item Gaussian distribution: $l_i \sim N(\mu, \sigma^2)$. The threshold is \begin{equation}
\label{eq:gaussian}
    t = \mu + \sqrt{2}\sigma(\mathrm{erf}^{-1}(1-2p)) + \sigma^2
\end{equation}
where $\mathrm{erf}(\cdot)$ denotes the error function\citep{andrews1998special}.

\item Uniform distribution: $l_i \sim U(M - a, M)$, The threshold is 
\begin{equation} \label{eq:uniform}
    t = M - \ln\left[\frac{1}{1-p(1-e^{-a})}\right]
\end{equation}

\end{enumerate}

The detailed derivations of these expressions are provided in Appendix~\ref{app:thres_cal}. These analytical expressions provide valuable insights into the relationship between the logit distribution and the sampling threshold. While the actual logit distribution in language models may be more complex, these fundamental cases serve as important theoretical benchmarks and can guide the design of more efficient sampling algorithms.
\end{remark}

\subsection{Understanding the Real Distribution of Logits}

\subsubsection{Noisy Region}

As demonstrated in Figure~\ref{fig:logits_distribution}, the vast majority of tokens exhibit logits that follow a Gaussian distribution. Since the probabilities corresponding to these logits are typically negligible and have been traditionally considered noise in previous works, we designate this region as the \textit{noisy region}. This characterization aligns well with statistical intuition, where Gaussian distributions often indicate the presence of random noise in a system.

When the margin between this noisy region and the informative region narrows, the noise-derived probabilities -- after softmax transformation -- tend to interfere with the model's generation process, potentially degrading the output quality. This phenomenon is particularly evident in high-temperature sampling scenarios, where all current non-deterministic sampling algorithms perform poorly. We attribute this degradation to the diminished gap between the two regions under high temperatures, allowing the noise distribution to dominate the probability landscape. We identify several key factors contributing to this noise:

\paragraph{Training Data Noise} The inherent noise and variations in the training data naturally propagate through the model's learning process, ultimately manifesting as part of the Gaussian distribution in the logit space.

\paragraph{Regularization Effects} Various regularization techniques employed in the training process, while crucial for preventing overfitting, have the side effect of assigning small but non-zero probabilities to semantically irrelevant tokens in the vocabulary space.

\paragraph{Noise of Silence} The model's architectural constraint of assigning finite values to irrelevant tokens (which ideally should be ``silent'' with $-\infty$ logits) results in a distinctive noise pattern.  This is an inherent flaw of the softmax function, which is widely discussed in \citep{miller2023attention, xiao2023efficient}.

While addressing these noise sources fundamentally lies beyond the scope of this paper, we believe these insights could contribute to improving training procedures in future work. In this paper, we will focus on how to eliminate these noisy tokens during the inference.

\subsubsection{Informative Region}

As shown in Figure~\ref{fig:logits_distribution}, a small subset of tokens accounts for the majority of the probability mass. This concentration suggests that large language models possess specific knowledge about these tokens, hence we designate this as the \textit{informative region}.

Due to the limited number of tokens in this region, it is challenging to make definitive claims about the underlying distribution. However, recent empirical observations regarding sampling methods have provided interesting insights into this region's characteristics.

Notably, we have observed that the \minp sampling method yields significant improvements in generation quality. This approach establishes a baseline probability threshold $p$ and eliminates all probability values below $p_{max} \cdot p$, where $p_{max}$ is the maximum probability. While \citet{nguyen2024min} derived this method empirically, we have made a surprising theoretical discovery:

\begin{theorem}
For logits following a uniform distribution, min-p sampling is equivalent to top-$(1-p)$ sampling.
\end{theorem}

\begin{proof}
    We first put \minp into the logits space. For $l_i\sim U(-\infty, M)$, the threshold is 
    \begin{equation}
    t = \ln(s \cdot p_{max}\cdot p) = \ln(e^{M}\cdot p) = M + \ln p
    \end{equation}

    Considering Eq.~(\ref{eq:uniform}), since $l_i\sim U(-\infty, M)$, $a = -\infty$, so the threshold for top-$(1-p)$ is the exactly same:

    \begin{equation}
        t = M - \ln\left[\frac{1}{1-(1-p)}\right] = M + \ln p
    \end{equation}
\end{proof} This analysis reveals that despite min-$p$'s claimed adaptiveness, it essentially performs a static truncation in the logits space. Furthermore, the effectiveness of \minp sampling suggests that the informative region approximately follows a \emph{uniform} distribution. 

\subsection{Determine the Boundary}

A natural approach to distinguish between informative and noisy regions would be treating the informative region as outliers of the noisy distribution. Under this assumption, conventional methods such as the $\mu + 3\sigma$ rule \citep{kazmier2009schaums} could be applied to determine the boundary. However, our empirical observations suggest that this approach may not be optimal for the task at hand.

We define the $\sigma$-distance as the number of standard deviations between the maximum probability and the mean value of the distribution, \ie, $\sigma\text{-distance} = \frac{M - \mu}{\sigma}$, where $M$ is the maximum logit, following the same notation as in the uniform distribution case.

\begin{figure}[htbp]
    \centering
    \begin{subfigure}[t]{0.48\linewidth}
        \centering
        \includegraphics[scale=0.5]{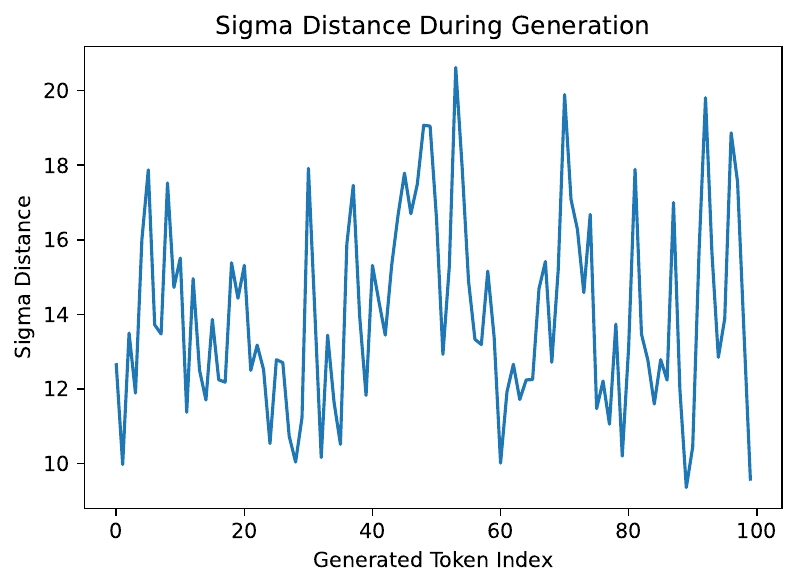}
        \caption{$\sigma$-distance during generation}
        \label{fig:sigma_dis}
    \end{subfigure}
    \hfill 
    \begin{subfigure}[t]{0.48\linewidth}
        \centering
        \includegraphics[scale=0.5]{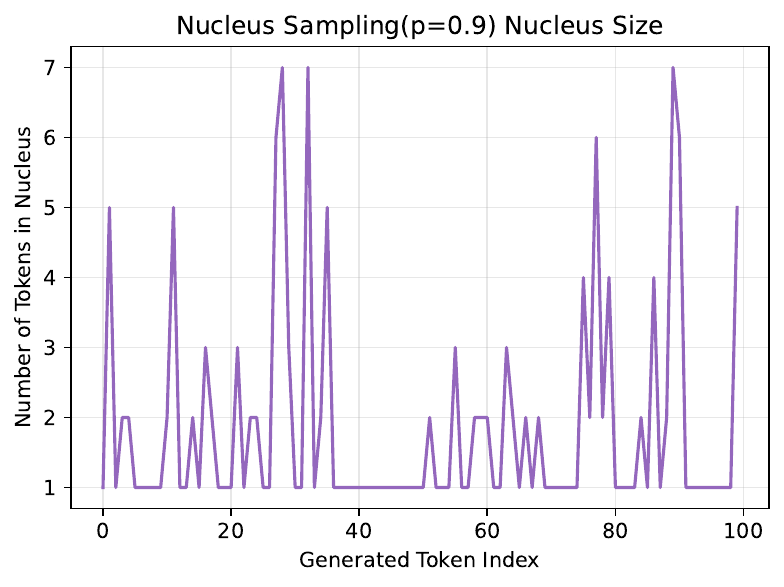}
        \caption{Nucleus size during generation}
        \label{fig:nucleus_sizes}
    \end{subfigure}
    \caption{Relationship between $\sigma$-distance and nucleus size when temperature $T=1$ during generation. The plots demonstrate an inverse relationship: high $\sigma$-distances correspond to small nucleus sizes (mostly 1), while lower $\sigma$-distances (though remaining above 10) correlate with larger nucleus sizes.}
    \label{fig:sigma_dis_nucleus_size}
\end{figure}

As illustrated in Figure~\ref{fig:sigma_dis}, the distance between the maximum probability and the mean value consistently surpasses $10\sigma$, exhibiting substantial fluctuations during the generation. In conjunction with Figure~\ref{fig:nucleus_sizes}, an intriguing pattern emerges: larger nucleus sizes correlate with lower $\sigma$-distances. This observation contradicts traditional outlier detection approaches such as the $\mu+3\sigma$ criterion, where higher $\sigma$-distances should theoretically encompass more tokens. This finding strongly suggests that informative tokens should not be treated as outliers of noisy tokens. In fact, a higher $\sigma$-distance indicates the model's strong confidence in its output, resulting in a smaller nucleus size. 

This leads us to a fundamental reversal of perspective: instead of treating the \textbf{minority as outliers} from the Gaussian distribution, we should recognize the \textbf{majority as outliers} from the uniform distribution. This counterintuitive shift challenges the conventional assumption of outlier detection, where anomalies are typically rare events. In our case, the tokens in noisy regions are essentially irrelevant candidates that should be eliminated from consideration, despite their numerical majority.

To implement this perspective, we introduce the top-$n\sigma$ algorithm. Our algorithm begins from the maximum value and extends downward, using the standard deviation of the distribution to dynamically adjust the boundary. Specifically, we capture a region that extends $n\sigma$ below the maximum value, where $n$ is determined empirically to balance between preserving informative tokens and excluding the uniform noise. In practice, we find $n=1.0$ achieves satisfactory performance.

\section{Algorithm Description}

Our method introduces a statistical threshold to filter candidate tokens before sampling. Algorithm \ref{alg:sampling} outlines the main steps of our approach. 

\begin{algorithm}[ht]
\caption{Top-$n\sigma$ Sampling}
\label{alg:sampling}
\begin{algorithmic}[1]
\State \textbf{Input:} Input context $x$, temperature $T$, threshold multiplier $n$ 
\State \textbf{Output:} Next token
\State Compute logits $l = \text{LLM}(x)$
\State Scale logits: $l' = l/T$
\State Calculate $M = \max(l')$ and $\sigma = \text{std}(l')$
\State Create mask: $m_i = \begin{cases} 
    1 & \text{if } l'_i \geq M - n\sigma \\
    0 & \text{otherwise}
\end{cases}$
\State Apply mask: $l'_i = \begin{cases} 
    l'_i & \text{if } m_i = 1 \\
    -\infty & \text{otherwise}
\end{cases}$
\State $p = \text{softmax}(l')$
\State Sample token from distribution $p$
\end{algorithmic}
\end{algorithm}

The algorithm is computationally efficient as it operates directly on logits without requiring extra softmax transformation or sorting. The core operations (\texttt{max} and \texttt{std}) are highly optimized for modern GPU implementations. The execution time typically remains within tens of microseconds, with the kernel launch overhead accounting for the majority of the computational cost. To better understand its behavior, we analyze two important aspects. First, we examine the probability mass of selected tokens in Section~\ref{sec:nucleus}, which provides a probabilistic view of our statistical filtering process. More importantly, we prove that our method selects a consistent number of tokens regardless of the temperature parameter in Section~\ref{sec:temp}. This temperature invariance property ensures stable behavior across different sampling temperatures, while still maintaining sufficient diversity in the candidate pool.

\subsection{Nucleus Mass} 
\label{sec:nucleus}

While the exact mixing mechanism of distributions in the logits space remains unknown, we focus our analysis on two boundary cases: (1) when the distribution approximates Gaussian, which typically occurs in high-temperature scenarios, and (2) when the distribution tends toward uniformity, as observed in low-temperature settings or when the model exhibits high confidence (usually accompanied by larger standard deviations). These two cases provide theoretical bounds for analyzing our method's behavior in practice.

In the Gaussian case, $l\sim N(\mu, \sigma^2)$, considering Eq.~(\ref{eq:gaussian}), we obtain the following equation:
\begin{equation}
   M-n\sigma = \mu + \sqrt{2}\sigma(\mathrm{erf}^{-1}(1-2p)) + \sigma^2
\end{equation}

Thus, 
\begin{equation}
    p = \frac{1}{2}\left[1 -\textrm{erf}\left(\frac{M-\mu-n\sigma-\sigma^2}{\sqrt{2}\sigma} \right)\right]
\end{equation}

A key property of this formulation is that as $\sigma$ approaches zero, the term inside $\textrm{erf}$ approaches positive infinity, causing $p$ to converge to zero. This effectively prevents the inclusion of tokens from the noise region, providing a natural safeguard against potential instabilities in the sampling process.

In the uniform case, $l\sim U(M-a, M)$, considering Eq.~(\ref{eq:uniform}), we have:

\begin{equation}
    M-n\sigma = M - \ln\left[\frac{1}{1-p(1-e^{-a})}\right]
\end{equation}

Thus, 

\begin{equation}
    p = \frac{1-e^{-n\sigma}}{1-e^{-a}}
\end{equation}

To determine an appropriate value for $a$, we use the fact that the overall standard deviation of the logits is $\sigma$. Given that the variance of a uniform distribution is $a^2/12$, we have $a^2/12 \leq \sigma^2$ to maintain consistency with the observed standard deviation. This provides us with an upper bound $a\leq 2\sqrt{3}\sigma$. Hence, 

\begin{equation}
    p \geq \frac{1-e^{-n\sigma}}{1-e^{-2\sqrt{3}\sigma}}
\end{equation}

For typical parameter values, such as $n=1$ and $\sigma=1.9$, we obtain a lower bound of 0.85, indicating that our top-$n\sigma$ algorithm effectively preserves the informative tokens. This analysis also provides valuable guidance for choosing the hyperparameter $n$: it should not only be positive but also remain below $2\sqrt{3}\approx 3.46$, due to $p\leq 1$. Exceeding this upper threshold would likely result in the inclusion of noisy tokens, potentially degrading the algorithm's performance.

\subsection{Temperature Invariance}
\label{sec:temp}

A key property of our sampling method is its temperature invariance. Specifically, the set of candidate tokens remains constant regardless of the temperature value used. 

\begin{theorem}
For any temperature $T > 0$, the nucleus of top-$n\sigma$ remains invariant.
\end{theorem}

\begin{proof}
Consider any token $i$ and temperature $T > 0$. Let $l_i$ be the original logit. After temperature scaling, we have $l'_i = l_i/T$ for all tokens. For any given token $i$:

\begin{align*}
M' &= \max_j(l'_j) = \frac{M}{T}\\
\sigma' &= \sqrt{\frac{1}{N}\sum_j(l'_j - \mu')^2} = \sqrt{\frac{1}{N}\sum_j(l_j/T - \mu_j/T)^2} = \frac{\sigma}{T}
\end{align*}

Token $i$ is selected if and only if $l'_i \geq M' - n\sigma'$. Substituting:
\begin{align*}
l'_i \geq M' - n\sigma' \iff 
\frac{l_i}{T} \geq \frac{M}{T} - \frac{n\sigma}{T} \iff 
l_i \geq M - n\sigma
\end{align*}

This final condition is independent of $T$. Therefore, for any token $i$, its inclusion in the selected set is determined by the same condition regardless of temperature.
\end{proof}

This temperature invariance property distinguishes our method from other common sampling approaches. For \topp and \minp sampling, the selected token set varies with temperature. As temperature increases, the logit distribution tends toward the Gaussian noise part, causing \topp and \minp to include more noise tokens in their sampling pool. While \topk sampling does maintain temperature invariance, it uses a fixed $k$ value, which merely shifts the problem: a static $k$ cannot dynamically distinguish between valid tokens and noise tokens across different contexts.

\paragraph{Nuanced Exploration Control}
With top-$n\sigma$, the exploration is decomposed into two distinct parameters. The parameter $n$ controls the size of the nucleus, while temperature solely regulates how to explore within the nucleus. This mechanism acts as a soft masking approach, effectively eliminating potentially irrelevant tokens while still allowing for a controlled degree of exploration. Such decoupling provides a more nuanced control over the sampling process: the parameter $n$ determines the boundary between valid and noise tokens, while the temperature parameter fine-tunes the exploration strategy within the validated token space.

\section{Experiments}

Language models typically employ deterministic algorithms (or low-temperature sampling) for reasoning tasks, while using stochastic approaches for generating diverse responses. Although diversity is often considered a key advantage of stochastic methods, its metrics remain ambiguous and unreliable. In this section, we evaluate our top-$n\sigma$ algorithm across four reasoning-focused datasets. Our experimental results show that top-$n\sigma$ achieves competitive performance against various baselines, including greedy decoding, and exhibits slight improvements at moderately higher temperatures ($T=1.5$), which challenges the conventional preference for low-temperature sampling in reasoning tasks.

To further investigate whether our approach maintains sufficient exploration capabilities or not, we implemented repeated sampling \citep{brown2024large} experiments, a Monte Carlo technique that generates multiple samples and determines the final answer through majority voting. Results show that top-$n\sigma$ does not compromise exploration ability while maintaining high accuracy.

\subsection{Setup}

\paragraph{Model} We use state-of-the-art open source language model LLaMA-3-8B-Instruct \citep{dubey2024llama} for evaluating our algorithm and vllm \citep{vllm} as our inference framework.  

\paragraph{Datasets} We conduct our experiments on four question answering datasets, especially focusing on reasoning: AQuA \citep{ling2017program}, MATH \citep{hendrycksmath2021}, GSM8K \citep{cobbe2021gsm8k}, GPQA \citep{rein2024gpqa}. The difficulty of the dataset ranges from elementary school math to doctoral-level problems. For all datasets, we transform them into an open-ended generation task. The model will generate a response and then we extract the answer from it, comparing it with the ground truth. We don't employ complex prompting techniques, simply asking the model to `reason' and answer the question in a prescribed form. More details can be found in Appendix \ref{app:exp_details}.

\paragraph{Baselines} We compare top-$n\sigma$ with various baselines, including \topk, \topp, \minp, temperature sampling. We present results between various temperatures. For all methods, the same set of hyperparameters is used across different temperatures to demonstrate the stability of their hyperparameters. For the baselines, we adopt their recommended hyperparameter settings based on previous work \citep{hewitt2022truncation, nguyen2024min} or practical guideline \citep{siml2024optimal}. The hyperparameter table is presented in Table~\ref{tab:hyperparameters}.

\begin{table}[ht]
\centering
\caption{Hyperparameter Settings}
\begin{tabular}{|c|c|}
\hline
\textbf{Hyperparameter} & \textbf{Value} \\
\hline
\topp & 0.9 \\
\hline
\minp & 0.1 \\
\hline
\topk & 20 \\
\hline
top-$n\sigma$ & 1.0 \\
\hline
\end{tabular}
\label{tab:hyperparameters}
\end{table}

For top-$n\sigma$, we proved that the theoretical bounds for parameter $n$ lie in $(0, 2\sqrt{3})\approx (0, 3.46)$. In practice, we find that $n=1.0$ provides consistently strong and robust performance across different scenarios without requiring task-specific tuning. Empirically, we recommend constraining $n$ larger than $0.5$ to avoid missing informative tokens. See Appendix~\ref{app:hyper} for more hyperparameter discussions.

\paragraph{Metrics} We employ two primary metrics for evaluation. For single-pass evaluation, we report the \textbf{Exact Match} (EM) score, which measures the percentage of model responses that exactly match the reference answers. For multiple-pass evaluation, we introduce \textbf{Maj@N}, where the model generates $N$ different responses for each query. The final prediction is determined through majority voting among these $N$ samples, after which the EM score is calculated between this consensus answer and the reference.

\subsection{Single-Pass Results}

\begin{table}[htbp]
\centering
\caption{Performance comparison of different sampling methods across datasets (Exact Match values in \%). \textbf{Bold} numbers indicate the best performance under each temperature setting, and \underline{underlined} bold numbers represent the highest score for each dataset. Notably, temperature = 0.0 represents greedy decoding, a \textbf{deterministic} algorithm rather than a sampling method.}

\label{tab:performance}
\small
\begin{tabular}{ll|ccccc}
\hline
Dataset & Method & \multicolumn{5}{c}{Temperature} \\
& & 0.0 & 1.0 & 1.5 & 2.0 & 3.0 \\
\hline
\multirow{5}{*}{GPQA} 
& Sample & \textbf{32.03} & 30.47 & 14.84 & 7.03 & 0.00 \\
& Top-$p$ & -- & \textbf{30.86} & 20.31 & 8.98 & 0.00 \\
& Top-$k$ & -- & 29.69 & 25.00 & 19.14 & 7.42 \\
& Min-$p$ & -- & 27.73 & 31.25 & 26.95 & 16.02 \\
& Top-$n\sigma$ & -- & 27.34 & \underline{\textbf{32.42}} & \textbf{27.73} & \textbf{25.00} \\
\hline
\multirow{5}{*}{GSM8K} 
& Sample & \textbf{81.25} & 76.95 & 21.48 & 0.00 & 0.00 \\
& Top-$p$ & -- & 78.52 & 66.02 & 0.00 & 0.00 \\
& Top-$k$ & -- & 75.78 & 62.11 & 21.88 & 2.34 \\
& Min-$p$ & -- & \textbf{80.47} & 76.56 & 66.41 & 14.84 \\
& Top-$n\sigma$ & -- & 78.52 & \underline{\textbf{82.03}} & \textbf{79.30} & \textbf{74.61} \\
\hline
\multirow{5}{*}{AQuA} 
& Sample & \textbf{36.61} & -- & -- & -- & -- \\
& Top-$p$ & -- & 39.76 & -- & -- & -- \\
& Top-$k$ & -- & 39.76 & 30.71 & 21.65 & -- \\
& Min-$p$ & -- & 37.80 & 37.01 & 33.07 & -- \\
& Top-$n\sigma$ & -- & \underline{\textbf{41.73}} & \textbf{40.94} & \textbf{40.16} & -- \\
\hline
\multirow{5}{*}{MATH} 
& Sample & \textbf{19.92} & -- & -- & -- & -- \\
& Top-$p$ & -- & 16.41 & -- & -- & -- \\
& Top-$k$ & -- & 14.06 & 10.55 & 3.91 & -- \\
& Min-$p$ & -- & 15.63 & 14.45 & 10.94 & -- \\
& Top-$n\sigma$ & -- & \underline{\textbf{20.31}} & \textbf{16.02} & \textbf{14.06} & -- \\
\hline
\end{tabular}
\end{table}

The experimental results provide compelling evidence for the superiority of our proposed top-$n\sigma$ sampling method across various reasoning tasks. While other methods show inconsistent performance. For instance, min-$p$ performs well on GSM8K but struggles on MATH500, and top-$k$ shows decent results on GPQA but underperforms on AQuA. Our top-$n\sigma$ method consistently achieves the best or near-best performance across all datasets. Notably, challenging the conventional wisdom that reasoning tasks benefit most from deterministic decoding, our stochastic sampling approach outperforms greedy decoding across all datasets. More remarkably, at higher temperatures where conventional methods suffer catastrophic degradation, top-$n\sigma$ maintains robust performance. For instance, at temperature 3.0, while standard sampling and top-$p$ completely fail on GPQA and GSM8K, our method still achieves 25.00\% and 74.61\% accuracy respectively. This stark contrast in high-temperature performance demonstrates the exceptional stability of our approach.

\paragraph{Optimal Temperature and Exploration}
An intriguing finding emerges from our experiments that challenges conventional wisdom. Despite the common guideline that reasoning tasks benefit most from deterministic (greedy) or low-temperature decoding \citep{xu2022systematic, zhu2023improving, renze2024effect}, our results demonstrate that slightly higher temperatures (typically around 1.5) yield optimal performance when combined with our controlled sampling strategy. This phenomenon aligns with our theoretical framework: under the strict nucleus control of top-$n\sigma$, moderate exploration in the token space actually enhances model performance rather than degrading it. 

\paragraph{Surpassing Greedy Decoding}
Perhaps the most significant achievement is that top-$n\sigma$ sampling consistently outperforms greedy decoding (temperature = 0.0) across all datasets, a feat that none of the other sampling methods could accomplish. This demonstrates that controlled stochastic sampling can outperform deterministic approaches even in reasoning-intensive tasks. Furthermore, these results represent only the baseline performance of our method - the incorporation of test-time scaling techniques promises to widen this performance gap substantially, potentially establishing an even more decisive advantage over greedy decoding.

\subsection{Multiple-Pass Results}

\begin{table}[htbp]
\centering
\caption{Maj@20 of Different Sampling Methods (\%)}
\begin{tabular}{l l | c c c c}
\hline
Dataset & Method & \multicolumn{4}{c}{Temperature} \\
 &  & 1.0 & 1.5 & 2.0 & 3.0 \\
\hline
\multirow{5}{*}{GSM8K} 
 & Sample & \textbf{90.63} & 75.00 & 0.00 & 0.00 \\
 & Top-p & 89.06 & 89.45 & 0.00 & 0.00 \\
 & Top-k & 89.45 & \textbf{91.41} & 62.89 & 2.73 \\
 & Min-p & 89.84 & 90.63 & 89.84 & 53.13 \\
 & Top-n$\sigma$ & \textbf{90.63} & \textbf{91.41} & \underline{\textbf{91.80}} & \textbf{90.23} \\
\hline
\multirow{5}{*}{GPQA}
 & Sample & 30.47 & 27.34 & 12.89 & 0.00 \\
 & Top-p & 30.08 & 27.34 & 12.89 & 0.00 \\
 & Top-k & \textbf{32.03} & 31.64 & 26.17 & 24.61 \\
 & Min-p & 30.47 & \underline{\textbf{33.20}} & 31.25 & \textbf{30.47} \\
 & Top-n$\sigma$ & 31.64 & \underline{\textbf{33.20}} & \textbf{32.42} & \textbf{30.47} \\
\hline
\multirow{5}{*}{AQuA}
 & Sample & - & - & - & - \\
 & Top-p & 44.88 & - & - & - \\
 & Top-k & 48.03 & 48.03 & 40.16 & - \\
 & Min-p & 44.09 & 51.18 & 47.64 & - \\
 & Top-n$\sigma$ & 47.64 & 46.06 & 49.61 & - \\
\hline
\multirow{5}{*}{MATH}
 & Sample & - & - & - & - \\
 & Top-p & 32.03 & - & - & - \\
 & Top-k & 31.25 & 20.70 & 12.50 & - \\
 & Min-p & 30.86 & 28.91 & 23.83 & - \\
 & Top-n$\sigma$ & 32.03 & 35.16 & 33.98 & - \\
\hline
\end{tabular}
\label{tab:maj_results}
\end{table}

As shown in Table~\ref{tab:maj_results}, all random sampling methods, enhanced with repeated sampling, outperform greedy decoding. Notably, the top-$n\sigma$ method demonstrates strong performance in most datasets, indicating that our approach maintains effective exploration capabilities. 

Regarding the impact of temperature in repeated sampling, we observe that slightly high temperatures have a more significant effect. Most methods exhibit better performance at a moderately higher temperature ($T=1.5$), as this allows for broader exploration of the solution space. However, while other methods suffer from severe performance degradation at higher temperatures due to interference from noisy tokens, the top-$n\sigma$ approach remains robust to such disturbances.

\paragraph{Performance Ceiling} 
The performance saturation phenomenon varies across different datasets. On GSM8K, all methods achieved remarkably high accuracy (90\%+) after sufficient sampling, suggesting that this mathematical reasoning dataset might be too elementary to differentiate between sampling strategies effectively. In contrast, GPQA presents a unique case where majority voting barely improved over single-sample performance.  Being a multiple-choice dataset with four options, GPQA has a baseline accuracy of 25\% for any strategy that can properly follow instructions. The observed accuracy of 30-32\% across different methods, only marginally above the random baseline, suggests that the model struggles substantially with this dataset. 

\section{Conclusion}

We have presented top-$n\sigma$, demonstrating both theoretical and empirical advantages over existing sampling methods. Our analysis reveals fundamental insights about logit distributions in large language models, challenging the conventional preference for greedy decoding in reasoning tasks. The method's temperature invariance and efficient computation make it particularly well-suited for emerging test-time scaling techniques. Beyond sampling strategies, our findings about the distinct separation between noisy and informative regions in logit space suggest potential improvements in model architecture and training procedures. Future work might explore how to leverage these statistical properties during training to enhance model performance. 

\bibliographystyle{unsrtnat}
\bibliography{references} 

\appendix

This preprint version is incomplete. We will supplement the missing experimental results and appendix in the full version later. 

\section{Theoretical Analysis}

\subsection{Proof of Theorem~\ref{thrm:sum_to_int}}
\label{app:proof_of_sum_to_int}

\subsection{Threshold Calculation}
\label{app:thres_cal}

\section{Experimental Details}
\label{app:exp_details}

\section{Hyperparameters}
\label{app:hyper}

\end{document}